\documentclass{article}

\usepackage[nonatbib,preprint]{neurips_2026}

\usepackage{amsmath,amssymb,amsthm}
\usepackage{graphicx}
\usepackage{url}
\usepackage{hyperref}
\usepackage[utf8]{inputenc}
\usepackage{algorithmic}
\usepackage{algorithm}
\usepackage{booktabs}
\usepackage{textcomp}
\usepackage{xcolor}

\newcommand{\degC}{$^{\circ}$C}

\newtheorem{theorem}{Theorem}
\newtheorem{corollary}[theorem]{Corollary}

\newtheorem{remark}{Remark}

\title{Excluding the Target Domain Improves Extrapolation:\\
Deconfounded Hierarchical Physics Constraints}

\author{Tsuyoshi Okita \\ Kyushu Institute of Technology \\ tsuyoshi@ai.kyutech.ac.jp}

\begin{document}

\maketitle

\begin{abstract}
Extrapolation to out-of-distribution conditions is a fundamental challenge
for physics-constrained deep generative models.
Existing methods apply physical constraints as a single static regularization term
uniformly across the generation process,
and address neither the hierarchical structure of physical laws
and the confounding variable problem.
We propose the \textbf{Deconfounded Hierarchical Gate (DHG)},
which serves as a \textit{diagnostic and control} mechanism:
it identifies \textit{when} and \textit{how strongly} temperature confounding contaminates
each constraint level, so that hierarchical gates reflect intrinsic physical inconsistency
rather than spurious temperature effects.
DHG combines counterfactual estimation via the do-operator
with backdoor adjustment to remove confounding,
then applies Coarse-to-Fine physical constraints progressively.
We report a counter-intuitive finding in pretraining:
\textbf{excluding} the target-domain data from pretraining
outperforms including it by 39\% in extrapolation performance
(RMSE 0.224 vs.\ 0.324).
This occurs because FNO learns domain-agnostic physical patterns
that transfer more effectively when the target domain is withheld.
On a lithium-ion battery temperature extrapolation benchmark
(trained at 24\degC, evaluated at 4.0--43.0\degC),
our method achieves RMSE = 0.215,
a 46\% improvement over the unconstrained baseline (Pure CFM: 0.397).
\end{abstract}

\section{Introduction}

Physics-constrained deep generative models face a fundamental challenge in extrapolating
to out-of-distribution (OOD) conditions.
When the condition changes---such as temperature, material type, or operating protocol---
the model encounters behaviors unseen during training, and prediction quality degrades rapidly.
This work establishes a \textbf{counter-intuitive yet principled result}:
\textbf{excluding the target condition from pretraining improves extrapolation
to that condition by 39\%} (RMSE 0.224 vs.\ 0.324)---
a finding that holds consistently across all 10 pretraining configurations tested
and is grounded in the theory of invariant risk minimization.
Intuitively, excluding the target forces the model to rely solely on
\textbf{condition-invariant physical structure}---such as temperature-dependent
ion transport, which is shared across all battery chemistries---
rather than target-specific shortcuts~\cite{peters2016causal}.
However, learning these invariant features correctly faces a fundamental obstacle:
the \textbf{confounding variable} problem.
In battery systems, for instance, temperature affects
multiple degradation mechanisms (reaction rate, film growth, and ion deposition).
When a physical constraint is violated, it is in principle impossible to distinguish
whether the cause is temperature or an intrinsic physical inconsistency---making
appropriate constraint application fundamentally difficult.
Existing methods~\cite{bastek2024physics,wang2022respecting} cannot address this confounding problem.
Bastek et al.\ apply a single PDE residual loss uniformly across all generation timesteps,
ignoring the hierarchical structure of physical laws,
while Causal PINN is limited to controlling a single PDE residual along the time axis.

To address both problems jointly, we propose the
\textbf{Deconfounded Hierarchical Gate (DHG)},
which resolves confounding via Pearl's do-operator~\cite{pearl2000causality}.
DHG estimates counterfactual violation scores by applying virtual-condition do-operators
to a validator $f_k$, then removes inter-level confounding via backdoor adjustment
to construct a gate based solely on \textbf{intrinsic physical inconsistency},
independent of temperature.
Based on this causal gate, physical constraints are applied in a Coarse-to-Fine
hierarchical manner: from global self-consistency (early generation, $\beta_1\!\approx\!0.2$)
to local self-consistency (late generation, $\beta_3\!\approx\!0.8$).
This work contributes:
\begin{itemize}
\item \textbf{Coarse-to-Fine hierarchical constraints}: We assign an independent
Sigmoid schedule to each constraint level according to the CFM generation timestep $t$,
enabling staged application of physical constraints.
\item \textbf{DHG (Deconfounded Hierarchical Gate)}: We propose, to our knowledge,
the first gate mechanism that brings Pearl's do-operator into the generation process itself.
DHG applies counterfactual estimation and backdoor adjustment at each generation timestep $t$
to diagnose and control temperature confounding,
constructing gates that reflect intrinsic physical inconsistency rather than spurious temperature effects.
This is the first method to causally deconfound hierarchical physical constraints
within a flow matching generative model.
This is supported by temperature-sensitive backdoor coefficients $\beta_{kj}$
and above-chance temperature discrimination accuracy
(NASA: 2.2$\times$ random, MICH\_EXP: 1.9$\times$ random).
\item \textbf{Target-exclusion pretraining principle}: We establish that
excluding the target domain from pretraining is not merely a dataset trick
but a principled strategy grounded in invariant risk minimization~\cite{peters2016causal}:
cross-domain diversity (NMC+LFP, no NASA) forces FNO(1) to learn
condition-invariant physical patterns (temperature-dependent ion transport)
rather than target-specific shortcuts,
yielding 39\% better extrapolation (RMSE 0.224 vs.\ 0.324).
This principle generalizes beyond batteries to any physics-constrained OOD problem
where the underlying PDE structure is shared across domains.
\item \textbf{Voltage waveform temperature extrapolation}: Extrapolating from training
at 24\degC\ only to 4.0--43.0\degC, we achieve RMSE = 0.215,
a 46\% improvement over the unconstrained baseline (Pure CFM RMSE: 0.397),
demonstrating that physics-guided generation with causal deconfounding
substantially outperforms unconstrained flow matching.
\end{itemize}

\section{Related Work}

\textbf{Neural Operators and Their Variants.}
Neural Operators learn mappings between function spaces~\cite{chen1995universal}, with FNO~\cite{li2020fourier} parameterizing the integral kernel in Fourier space for efficient PDE solving. Variants include GNO~\cite{li2020gno}, DeepONet~\cite{lu2021deeponet}, WNO~\cite{tripura2023wno}, U-FNO~\cite{wen2022ufno}, and VINO~\cite{eshaghi2025vino}. PINO~\cite{li2021pino} combines FNO with PDE residual losses, but its ``hierarchy'' refers to multi-scale resolution, not the priority ordering of physical laws addressed here; we apply constraints in a Coarse-to-Fine manner across generation timestep $t$, from global to local self-consistency.

\textbf{Conditional Flow Matching.}
Flow Matching~\cite{lipman2022flow} offers efficient training for continuous normalizing flows, and CFM extends this to condition-dependent generation, which we use for temperature-conditioned battery waveform generation. Unlike existing CFM methods, this work explicitly integrates physical constraints through frozen FNO(1) guidance.

\textbf{Physics-Informed Generative Models.}
Physics-Informed Diffusion Models~\cite{bastek2024physics} added a PDE residual loss to the diffusion model training objective---the closest prior work---but applies a single static constraint uniformly across all generation timesteps, ignoring the hierarchical structure of physical laws. Causal PINN~\cite{wang2022respecting} respects temporal causal ordering within a single PDE residual, which differs from our multi-level constraint hierarchy. Physics-Integrated VAE~\cite{takeishi2021physvae} and PITA~\cite{zhu2025pita} address physical consistency but do not handle confounding removal. HPC-FNO-CFM extends these methods with staged hierarchical constraints, causal deconfounding via DHG, and target-exclusion pretraining---three components not addressed by any prior work.

\textbf{Causal Deconfounding and Generative Models.}
DeCaFlow~\cite{almodovar2025decaflow} removes confounding in causal generative models
using the do-operator with proxy variables, targeting causal effect estimation (ATE/counterfactual MAE)
on static observational data with a known causal graph.
DHG differs in three ways:
(1) it operates \textit{within} the generation process at each timestep $t$,
not on static data;
(2) it targets physical constraint gate construction rather than causal effect estimation;
and (3) it combines confounding removal hierarchically with multi-level physical constraints.
SP-FM~\cite{almodovar2025spfm} achieves OOD generalization via conditional flow matching but without explicit physical constraint hierarchy.

\textbf{Battery Degradation Prediction.}
Existing studies~\cite{saha2008battery,roman2021machine,biggio2022dynaformer} are largely limited to prediction within training conditions and do not address temperature extrapolation beyond 20\degC. This work demonstrates that causal deconfounding and hierarchical physical constraints enable extrapolation from 24\degC\ to 4.0--43.0\degC.

\section{Proposed Method: HPC-FNO-CFM}

\begin{figure*}
\begin{center}
\includegraphics[width=0.9\textwidth]{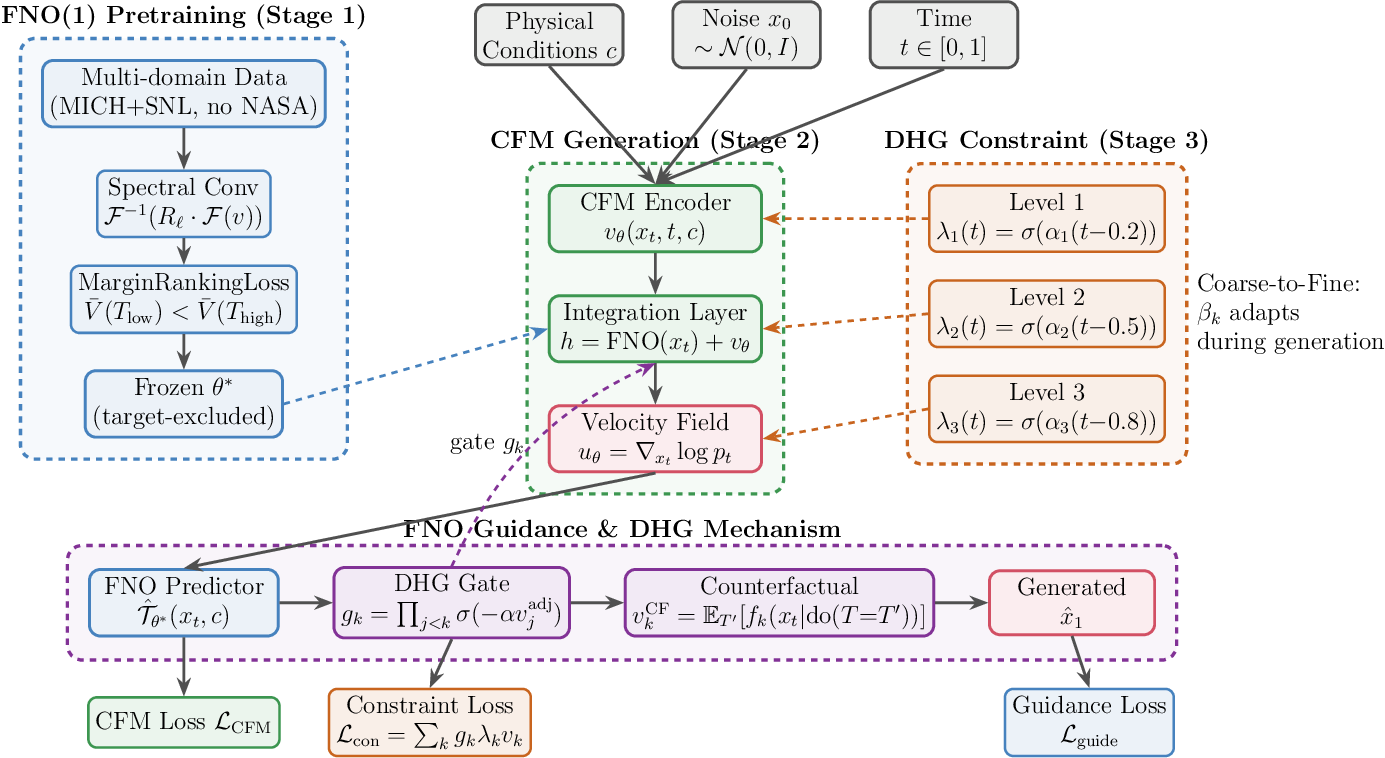}
\end{center}
\caption{Overview of the HPC-FNO-CFM framework.
\textbf{Left}: FNO(1) pretrained on multi-condition battery data (Stage~1)
using spectral convolution to learn condition-dependent physical patterns;
parameters are frozen after Stage~1.
\textbf{Center}: Condition-conditioned CFM generation network (Stage~2).
The frozen FNO(1) provides physical guidance to the CFM velocity field
through an Integration Layer, analogous to PDE guidance and neural operator guidance.
\textbf{Right}: Coarse-to-Fine constraint schedule via DHG (Stage~3),
$\lambda_k(t)=\sigma(\alpha_k(t-\beta_k))$ ($\beta_k \in \{0.2,\,0.5,\,0.8\}$);
each level activates at a different stage of generation.
\textbf{Bottom}: DHG mechanism.
Temperature confounding is removed via counterfactual violation scores
$v_k^{\text{CF}}=\mathbb{E}_{T'}[f_k(x_t\mid\mathrm{do}(T{=}T'))]$
and backdoor adjustment $\beta_{kj}$;
the DHG gate $g_k=\prod_{j<k}\sigma(-\alpha v_j^{\text{adj}})$
enforces hierarchical causality between constraint levels.}
\label{fig:architecture}
\end{figure*}

\subsection{Overall Design}

HPC-FNO-CFM is built on three core ideas (Fig.~\ref{fig:architecture}).
First, we pretrain \textbf{FNO(1)\footnote{%
FNO(1) denotes the Fourier Neural Operator applied to one-dimensional sequential data
(sequence length $L$). The ``(1)'' indicates the spatial dimension of the operator input;
for time series, the temporal axis constitutes this single dimension.}}
to learn condition-dependent physical patterns from multi-condition data.
FNO(1) is a neural operator based on spectral convolution in Fourier space~\cite{li2020fourier},
and is trained independently of the downstream task.
After pretraining, its parameters $\theta^*$ are frozen.
In Stages~2 and~3, the frozen FNO(1) provides \textbf{physical guidance}
to the generative model---analogous to PDE guidance~\cite{bastek2024physics}
and neural operator guidance (DeepONet~\cite{lu2021deeponet}, WNO~\cite{tripura2023wno})---
enabling the model to capture physical conditions that a single learned function cannot cover alone.
This preserves the learned physical patterns as an inductive bias (Theorem~\ref{thm:freeze}).
Second, we adopt a \textbf{three-stage training structure}:
Stage~1 trains FNO(1) to learn condition-dependent patterns;
Stage~2 trains CFM to learn conditional probability generation;
Stage~3 refines generated waveforms using a physical consistency loss.
Third, \textbf{Coarse-to-Fine hierarchical constraints} form the central novelty of this method.
The validator $f_k$ at each constraint level $k$ evaluates
not a specific PDE residual but a self-consistency score of $x_t$
at generation timestep $t$: $v_k = f_k(x_t) \in [0, 1]$.
Each of the $K$ constraint levels is assigned an independent Sigmoid schedule:
\begin{equation}
\lambda_k(t) = \sigma\!\left(\alpha_k\,(t - \beta_k)\right), \quad k = 1, \ldots, K
\label{eq:sigmoid_schedule}
\end{equation}
The $\beta_k$ values are initialized to $[0.2, 0.5, 0.8]$,
inducing activation of each level at the early, middle, and late stages
of the generation process.
The central novelty of this method is the \textbf{DHG},
which acts as a \textit{diagnostic and control} mechanism for temperature confounding.
DHG quantifies \textit{how much} temperature confounding contaminates
each constraint level at each generation timestep $t$,
so that the hierarchical gate $g_k$ reflects intrinsic physical inconsistency
rather than spurious temperature-driven variations.
Empirically, this is evidenced by temperature-sensitive backdoor coefficients $\beta_{kj}$
and above-chance temperature discrimination accuracy (Section~\ref{sec:dhg_analysis}).
DHG operates in two steps.
First, counterfactual violation scores are estimated
by applying the do-operator~\cite{pearl2000causality}
over a virtual temperature set $\mathcal{T}_{\text{cf}}$:
\begin{equation}
v_k^{\text{CF}} = \frac{1}{|\mathcal{T}_{\text{cf}}|}
\sum_{T' \in \mathcal{T}_{\text{cf}}} f_k(x_t \mid \mathrm{do}(T = T'))
\label{eq:cf_violation}
\end{equation}

Next, inter-level confounding is removed via backdoor adjustment:
\begin{equation}
v_k^{\text{adj}} = v_k^{\text{CF}} - \sum_{j < k} \beta_{kj} \, v_j^{\text{CF,sg}}
\label{eq:backdoor}
\end{equation}

Finally, a hierarchical gate is constructed from the deconfounded scores $v_k^{\text{adj}}$:
\begin{equation}
g_k^{\text{DHG}} = \prod_{j < k} \sigma\!\left(-\alpha\, v_j^{\text{adj,sg}}\right)
\label{eq:dhg}
\end{equation}
Since $g_k^{\text{DHG}}$ is constructed from temperature-deconfounded scores,
it reflects the hierarchical dependencies between constraint levels
rather than raw temperature effects.

\subsection{Stage 1: Physical Pattern Pretraining of FNO(1)}

FNO(1) uses spectral convolution layers~\cite{li2020fourier}
to learn condition-dependent physical patterns from multi-condition pretraining data
(see Appendix for the spectral convolution formulation).
Our implementation uses $k_{\max} = 16$ Fourier modes, width $d_v = 64$, and 3 layers.
For pretraining, we use a MarginRankingLoss
$\mathcal{L}_{\text{pt}} = \mathbb{E}_{c^-, c^+}[\max(0,\, \bar{y}(c^+) - \bar{y}(c^-) + m)]$---a contrastive loss enforcing known ordering between outputs across condition groups,
where $c^-, c^+$ are condition pairs with known ordering and $m > 0$ is the margin.
In battery experiments, this enforces $\bar{V}(T_{\text{low}}) < \bar{V}(T_{\text{high}})$,
enabling FNO(1) to learn temperature-dependent waveform patterns
independently of electrode material (see Experiment~3).
After pretraining, the parameters $\theta^*$ of FNO(1) are frozen.
In Stages~2 and~3, $\theta^*$ is not updated.
The theoretical justification for freezing is given in Theorem~\ref{thm:freeze}.

\subsection{Stage 2: Conditional Probability Generation via CFM}

In Stage~2, the physical condition $c \in \mathcal{C}$ is mapped to a
$d_h$-dimensional feature vector by an MLP.
CFM~\cite{lipman2022flow} learns the velocity field $v_\theta(x,t,c)$
of the generation ODE $dx/dt = v_\theta$,
using a U-Net structure with the squared error between $v_\theta$ and the target field as loss.
The frozen FNO(1) output $u_{\text{FNO}} = \widehat{\mathcal{T}}_{\theta^*}(x, c)$
is concatenated as an additional input via stop-gradient:
$v_\theta(x, t, c) = \text{UNet}(x,\; t,\; h_c,\; \text{sg}(u_{\text{FNO}}))$,
so that FNO(1) parameters $\theta^*$ receive no gradient updates in this stage.

\subsection{Stage 3: Physical Constraint Refinement}

The core of this method is the Coarse-to-Fine hierarchical constraint loss,
which progressively applies physical constraints according to
generation timestep $t$.
In contrast to Bastek et al.~\cite{bastek2024physics}, who apply a single static constraint
uniformly across all timesteps, we assign an independent Sigmoid schedule
(Eq.~\ref{eq:sigmoid_schedule}) to each constraint level $k$.

The DHG gate $g_k$ (Eq.~\ref{eq:dhg}) ensures that gradients from higher-level
constraints are passed only after lower-level constraints are satisfied,
where counterfactual scores (Eq.~\ref{eq:cf_violation}) are first deconfounded
via backdoor adjustment (Eq.~\ref{eq:backdoor}):
\begin{equation}
\mathcal{L}_{\text{constraint}} =
\sum_{k=1}^{K} g_k \cdot \lambda_k(t) \cdot \mathcal{L}_k
\label{eq:constraint_loss}
\end{equation}
where $\mathcal{L}_k$ is the constraint violation at Level $k$.
The total loss is $\mathcal{L}_{\text{total}} = \mathcal{L}_{\text{CFM}}
+ \gamma_{\text{fno}}\,\mathcal{L}_{\text{FNO}}
+ \gamma_{\text{con}}\,\mathcal{L}_{\text{constraint}}$ (Eq.~\ref{eq:constraint_loss})
with $\gamma_{\text{fno}} = 0.1$ and $\gamma_{\text{con}} = 0.01$.

The Coarse-to-Fine design is motivated by the structure of the CFM generation process
$x_t = t\,x_1 + (1-t)\,x_0$:
at $t \approx 0$ (early stage), $x_t$ is close to noise and global structure is determined;
at $t \approx 1$ (late stage), $x_t$ approaches clean data and local details are refined.

The constraint validators $f_k$ do not rely on domain-specific physical equations;
instead, they learn \textbf{self-consistency scores} at each generation timestep from data.
Each validator $f_k$ is a two-layer MLP with ReLU activations and Sigmoid output,
taking as input the flattened waveform $x_t$ concatenated with the temperature embedding:
$f_k: \mathbb{R}^{L \cdot d + d_T} \to [0,1]$.
The validators are trained jointly with Stage~3 using a contrastive loss
that encourages low scores on real data and high scores on generated data,
preventing the degenerate solution of outputting zero for all inputs.
The three levels evaluate consistency at different timescales:
(1) Level~1 ($\beta_1 \approx 0.2$): global self-consistency in the early generation stage
(e.g., overall discharge trend in voltage waveforms);
(2) Level~2 ($\beta_2 \approx 0.5$): mid-scale self-consistency in the middle stage
(e.g., condition-dependent patterns);
(3) Level~3 ($\beta_3 \approx 0.8$): local self-consistency in the late stage
(e.g., physical range of voltage values).

Each $\beta_k$ is optimized as a learnable parameter from its initial value $[0.2, 0.5, 0.8]$.
Experiments confirm that after training, values remain near $[0.24, 0.54, 0.83]$,
and that Coarse-to-Fine differentiation improves extrapolation RMSE
compared to zero initialization.

\begin{minipage}[t]{0.52\textwidth}\vspace{6pt}
The DHG gate causally enforces this hierarchical ordering,
preventing Level~2 and~3 gradients from advancing
before Level~1 self-consistency is achieved.

This design is in principle applicable to any domain where physical constraints
have a hierarchical priority ordering;
extension to other domains is left for future work.
  
The three-stage procedure is summarized in Algorithm~\ref{algorithm}.
Stage~1 pretrains FNO(1) on multi-condition data to learn condition-dependent
physical patterns; parameters $\theta^*$ are then frozen.
Freezing prevents catastrophic forgetting when downstream data is limited
($n_{\text{ds}} \ll n_{\text{pt}}$; Theorem~\ref{thm:freeze}).
Stage~2 trains CFM with frozen FNO(1) providing physical guidance.
Stage~3 refines the generated waveforms by applying
Coarse-to-Fine hierarchical constraints via DHG,
without updating the frozen FNO(1).
\end{minipage}
\hfill
\begin{minipage}[t]{0.47\textwidth}
\begin{algorithm}[H]
\caption{HPC-FNO-CFM Training}
\label{algorithm}
\begin{algorithmic}[1]
\STATE \textbf{Stage 1: Pretrain FNO(1)}
\FOR{each batch $(x^-, x^+, c^-, c^+) \in \mathcal{D}_{\text{pt}}$}
  \STATE $\theta \leftarrow \text{Adam}(\nabla_\theta \mathcal{L}_{\text{pt}})$
\ENDFOR
\STATE $\theta^* \leftarrow \theta$;\quad \textbf{Freeze} $\theta^*$
\STATE \textbf{Stage 2: Train CFM} (fix $\theta^*$)
\FOR{each batch $(x_1, c) \in \mathcal{D}_{\text{ds}}$}
  \STATE $u_{\text{FNO}} = \text{sg}(\widehat{\mathcal{T}}_{\theta^*}(x_t, c))$
  \STATE $\psi \leftarrow \text{Adam}(\nabla_\psi \mathcal{L}_{\text{CFM}})$
\ENDFOR
\STATE \textbf{Stage 3: Refine with constraints} (fix $\theta^*$)
\FOR{each batch $(x_1, c) \in \mathcal{D}_{\text{ds}}$}
  \STATE $\psi \leftarrow \text{Adam}(\nabla_\psi \mathcal{L}_{\text{total}})$
\ENDFOR
\STATE \textbf{Return:} $\theta^*$, $\psi$
\end{algorithmic}
\end{algorithm}
\end{minipage}

\section{Experiments}

\subsection{Task: Temperature Extrapolation in Battery Degradation}

The primary task of this work is a temperature extrapolation problem:
predicting battery degradation waveforms at unobserved temperatures
(4.0\degC--43.0\degC) from data at a single training temperature (24\degC).
We use the NASA battery dataset~\cite{fricke2023accelerated}.

\textbf{Training data:}
Batteries B0005, B0006, and B0007 (24\degC, 168 cycles each) are used
for temperature extrapolation experiments
(B0005--B0007 also used for ablation and interpolation experiments).
Five evaluation groups cover a range of temperatures:
Near (10.6\degC, B0042--B0044, $\Delta T = -13.4$\degC),
Low1 (4.0\degC, B0045--B0048, $\Delta T = -20.0$\degC),
Low2 (4.0\degC, B0053--B0056, $\Delta T = -20.0$\degC),
High1 (38.9\degC, B0038--B0040, $\Delta T = +14.9$\degC),
and High2 (43.0\degC, B0029--B0032, $\Delta T = +19.0$\degC).

\textbf{Input modalities:}
(1) Capacity scalar $Q(n)$: integrated discharge capacity (sequence length 50).
(2) Voltage waveform $V(t)$: discharge voltage profile resampled to 50 points,
normalized by per-cycle min/max values (relative normalization to $[0,1]$).

\textbf{Evaluation metrics:}
RMSE (primary), FID (distribution quality),
and temperature discrimination accuracy
(CFM score-based; random baseline = 0.167).

\subsection{Experimental Setup}\label{sec:exp_settings}

\textbf{Model:}
FNO(1): 3-layer FNO with 16 Fourier modes and width 64, 9.5M parameters total.
CFM: U-Net velocity field conditioned on temperature condition vector.

\textbf{Datasets:}
The primary downstream dataset is the NASA battery dataset~\cite{fricke2023accelerated}
(LiCoO\textsubscript{2}/graphite, 18650 cylindrical cells, 4.0--43.0\degC).
For DHG parameter analysis (Experiment~6), we additionally use
\textbf{MICH\_EXP}: the Michigan experimental subset of the BatteryLife dataset~\cite{batterylife2024},
comprising NMC cells cycled at $-5$ to $45$\degC\ (temperature range 50\degC),
which provides a wider temperature span than NASA (39\degC) and thus stronger confounding signal.
Pretraining uses MICH (NMC) and SNL (LFP) data, excluding NASA.

\textbf{Condition encoding $\phi(c)$:}
Temperature $T$ is encoded as a 3-dimensional condition vector:
\begin{equation}
\phi(T) = \left[\frac{T - T_{\text{ref}}}{\sigma},\;
-\frac{E_a}{R}\!\left(\frac{1}{T_K} - \frac{1}{T_{\text{ref},K}}\right),\;
e^{-E_a/RT_K}\right]
\end{equation}
where $E_a = 0.6$ eV, $T_{\text{ref}} = 24$\degC, $\sigma = 20$\degC.

FNO(1) is pretrained with MarginRankingLoss contrastive loss enforcing
$\bar{V}(T_{\text{low}}) < \bar{V}(T_{\text{high}})$ between temperature groups
(learning rate $10^{-4}$, 500 epochs, WSD scheduler).
Downstream training uses learning rate $10^{-3}$ for 300 epochs
(Stage~1 skipped; Stage~2: CFM training; Stage~3: physical constraint refinement).
Baselines are: \textbf{Scratch} (FNO(1) randomly initialized, trained jointly with hierarchical constraints);
\textbf{Finetune} (pretrained FNO(1) used as initialization with all parameters updated);
and \textbf{Pure CFM} (Stage~2 only, no FNO guidance).

\subsection{Experiment 2: Voltage Waveform Temperature Extrapolation}

Table~\ref{tab:volt_extrap} shows the main temperature extrapolation results
for voltage waveforms under five evaluation conditions.
\begin{table*}[h]
\centering
\caption{Voltage waveform temperature extrapolation results (RMSE, lower is better).
Pretraining: MICH+SNL (NMC+LFP, no NASA data).
Training: 24\degC\ only (B0005/B0006/B0007).
Gradient clipping added to Stages~1/2.
$\dagger$: corrected learning rate scheduling (lr reset after Stage~2).}
\label{tab:volt_extrap}
\begin{tabular}{llrrrrrr}
\toprule
Method & Freeze layers & Near & Low1 & Low2 & High1 & High2 & Mean \\
\midrule
Pure CFM (baseline) & -- & 0.405 & 0.403 & 0.375 & 0.451 & 0.351 & 0.397 \\
Scratch             & 3  & 0.521 & 0.475 & 0.431 & 0.527 & 0.444 & 0.480 \\
Freeze (1 layer)    & 1  & 0.259 & 0.205 & 0.219 & 0.229 & 0.240 & 0.230 \\
Freeze (2 layers)   & 2  & 0.260 & 0.220 & 0.217 & 0.253 & 0.178 & 0.226 \\
\textbf{Freeze (3 layers, proposed)} & \textbf{3} &
\textbf{0.259} & \textbf{0.205} & \textbf{0.231} &
\textbf{0.229} & \textbf{0.198} & \textbf{0.224} \\
Freeze (2 layers)$\dagger$ & 2 & -- & -- & -- & -- & -- & \textbf{0.215} \\
\bottomrule
\end{tabular}
\end{table*}
With voltage waveform input, the 3-layer frozen FNO achieves RMSE = 0.224,
a \textbf{2.1$\times$ improvement} over scratch (0.480).
FID also improves dramatically: 1.41 (proposed) vs.\ $>$10 (scratch).
Notably, pretraining on \textbf{MICH (NMC) + SNL (LFP) without NASA data}
outperforms pretraining that includes NASA data
(see Table~\ref{tab:pretrain_effect}).
Among freeze layer counts, 1--3 layers yield similar extrapolation performance,
with 3-layer freezing giving the most stable results.
A follow-up experiment with corrected learning rate scheduling (Table~\ref{tab:volt_extrap}, $\dagger$) achieves RMSE = 0.215,
a 46\% improvement over pure CFM (0.397)
(see Section~\ref{sec:dhg_analysis} for DHG parameter analysis under this setting).

\subsection{Experiment 3: Effect of Cross-Condition Pretraining}

Table~\ref{tab:pretrain_effect} compares downstream RMSE
under different pretraining data combinations,
revealing the counter-intuitive finding that excluding the target domain improves generalization.
Four key findings emerge:
(1) \textbf{Pretraining loss and downstream RMSE are non-monotonically related}:
SNL alone achieves low pretraining loss (0.00446)
but limited extrapolation performance.
(2) \textbf{Diversity of temperature conditions is necessary}:
the MICH+SNL combination performs best;
single-dataset pretraining tends to fail in extrapolation.
(3) \textbf{Cross-condition transfer is effective}:
NMC+LFP pretraining improves temperature extrapolation on the NASA (LiCoO2) target.
FNO(1) learns temperature-dependent voltage waveform patterns
that are independent of electrode material.
(4) \textbf{Mixed waveforms in CALB}:
approximately 42\% of CALB data contains charge waveforms
(whereas NASA, SNL, and MICH contain discharge only),
which is one cause of RMSE degradation and pure\_cfm divergence
in conditions that include CALB.
Adding gradient clipping ($\|\nabla\|_{\max} = 1.0$) to Stages~1 and~2
suppresses divergence (CALB+MICH: pure\_cfm RMSE 0.721$\to$0.406).

\begin{table*}[h]
\centering
\caption{Effect of pretraining data on downstream RMSE.
All models use the same downstream settings (lr=1e-3, 300 epochs, 3-layer freeze).
$\dagger$: CALB contains mixed charge/discharge waveforms (see text).}
\label{tab:pretrain_effect}
\begin{tabular}{llrr}
\toprule
Pretraining data & Electrode type & Pretrain loss & RMSE \\
\midrule
\textbf{MICH+SNL (no NASA)} & \textbf{NMC+LFP} & \textbf{0.00979} & \textbf{0.224} \\
SNL only (no NASA)          & LFP              & 0.00446 & 0.228 \\
NASA + MICH+SNL             & LiCoO2+NMC+LFP   & 0.01051 & 0.302 \\
CALB+SNL (no NASA)$\dagger$ & LiCoO2+LFP       & 0.00712 & 0.257 \\
NASA + CALB+SNL$\dagger$    & LiCoO2+LFP       & 0.00997 & 0.254 \\
No NASA + all BL$\dagger$   & All types        & 0.01059 & 0.351 \\
CALB+MICH (no NASA)$\dagger$& LiCoO2+NMC       & 0.00832 & 0.322 \\
NASA + all BL$\dagger$      & LiCoO2 mixed     & 0.01141 & 0.416 \\
NASA + CALB+MICH$\dagger$   & LiCoO2+NMC       & 0.00941 & 0.429 \\
MICH only (no NASA)         & NMC              & 0.00552 & 0.228 \\
\bottomrule
\end{tabular}
\end{table*}

\subsection{Discussion: Why Does NASA-Excluded Pretraining Outperform?}
\label{sec:discussion_cross_chemistry}

The most important and counter-intuitive finding of this experiment is that
\textbf{excluding the target domain (NASA: LiCoO2) from pretraining
yields better temperature extrapolation performance than including it}
(RMSE 0.224 vs.\ 0.302; Table~\ref{tab:pretrain_effect}).
From a physical perspective, the temperature dependence of voltage waveforms
originates not from the electrode material type but from
\textbf{the temperature-dependent transport of Li ions in the electrolyte},
a mechanism common across LiCoO2 (NASA), NMC (MICH), and LFP (SNL).
When NASA data is included, the LiCoO2-specific discharge curve shape
dominates FNO(1)'s learning, obscuring the condition-invariant transport patterns.
With NMC+LFP (no NASA), only the structure shared across these two chemically
distinct systems---namely, temperature-dependent ion transport---is learned,
resulting in improved transfer to NASA.
From a machine learning perspective, this finding is consistent with insights
from self-supervised learning~\cite{chen2020simclr}:
``diversity of source domains determines the generality of representations.''
From the viewpoint of invariant risk minimization~\cite{peters2016causal},
excluding NASA promotes learning of causal features that are invariant
across environments (temperature-dependent transport).
Furthermore, NASA data contains temperature confounding (Simpson's Paradox: $r=-0.204$ overall, $r=+0.050$ conditional), which is inherited when NASA is included in pretraining.

\subsection{Experiment 4: Hyperparameter Search}

Tables~\ref{tab:hp_search} and~\ref{tab:pfd} show the hyperparameter search results
and the correlation between FID and RMSE, respectively.
Learning rate $10^{-3}$ with 300 epochs was found to be optimal.
Lower learning rates ($10^{-4}$) converge too slowly,
while higher rates ($2\times10^{-3}$) cause divergence.
This configuration was adopted for all subsequent experiments.

\begin{table*}[h]
\centering
\begin{minipage}[t]{0.45\textwidth}
\centering
\caption{Hyperparameter search results.}
\label{tab:hp_search}
\begin{tabular}{lrr}
\toprule
Configuration & lr & RMSE \\
\midrule
lr=1e-3, ep=300 & $10^{-3}$ & \textbf{0.285} \\
lr=1e-3, ep=150 & $10^{-3}$ & 0.296 \\
lr=1e-4, ep=500 & $10^{-4}$ & 0.396 \\
lr=5e-4, ep=200 & $5\times10^{-4}$ & 0.463 \\
lr=2e-3, ep=100 & $2\times10^{-3}$ & diverged \\
\bottomrule
\end{tabular}
\end{minipage}
\hfill
\begin{minipage}[t]{0.50\textwidth}
\centering
\caption{Correlation between FID and RMSE. Lower FID corresponds to lower RMSE.}
\label{tab:pfd}
\begin{tabular}{lrr}
\toprule
Condition & FID (mean) & RMSE \\
\midrule
Freeze/MICH+SNL  & \textbf{1.56} & \textbf{0.219} \\
Freeze/CALB+MICH & 1.56          & 0.227 \\
Freeze/CALB+SNL  & 1.53          & 0.229 \\
Freeze/NASA+all BL & 4.07        & 0.282 \\
Scratch          & $>$10         & 0.501 \\
\bottomrule
\end{tabular}
\end{minipage}
\end{table*}

\subsection{Experiment 5: Physical Feature Distance (PFD) as an Auxiliary Metric}

FID ranks all conditions in the same order as RMSE (Table~\ref{tab:pfd}),
demonstrating that it functions as a valid quality metric
for physically meaningful temperature-conditioned generation.
We refer to this as the \textbf{Physical Feature Distance (PFD)}.

\subsection{Experiment 6: Analysis of DHG Parameters}
\label{sec:dhg_analysis}

DHG is designed as a \textbf{diagnostic and control} mechanism
for temperature confounding within the hierarchical constraint process.
It detects \textit{when} and \textit{how strongly}
temperature confounding affects each constraint level,
constructing gates that reflect intrinsic physical inconsistency rather than spurious temperature effects.
We validate this design intent through three analyses:
learned backdoor coefficients $\beta_{kj}$,
per-timestep waveform confounding patterns,
and temperature discrimination accuracy from generated waveforms.
First, we analyzed the convergence patterns of backdoor coefficients $\beta_{kj}$.
NASA and MICH\_EXP show markedly different patterns after training
(Table~\ref{tab:dhg_betas}).
\begin{table}[h]
\centering
\caption{Learned backdoor coefficients $\beta_{kj}$ (clamp upper bound 0.8).
$\beta_{kj}$: confounding coefficient from Level $j$ to Level $k$.}
\label{tab:dhg_betas}
\begin{tabular}{lrrrr}
\toprule
Dataset & $\beta_{21}$ & $\beta_{31}$ & $\beta_{32}$ & scale1 \\
\midrule
NASA (4--43\degC, range 39\degC)       & 0.000 & 0.673 & 0.986 & $-33.9$ (diverged) \\
MICH\_EXP ($-5$--45\degC, range 50\degC) & 0.501 & 0.501 & 0.501 & $+0.13$ (normal) \\
\bottomrule
\end{tabular}
\end{table}
In NASA, $\beta_{21} \approx 0$ (no Level~2$\leftarrow$Level~1 confounding detected),
whereas in MICH\_EXP confounding is detected across all level pairs
(all $\beta \approx 0.5$, at the clamp upper bound).
scale1 is $-33.9$ (diverged) in NASA and $+0.13$ (normal) in MICH\_EXP.
The larger the temperature range, the stronger the confounding and
the more the confounding regularization loss is activated,
which naturally suppresses divergence of the scale parameters.
Next, we analyzed per-timestep voltage waveform differences between
low temperature (4\degC) and high temperature (43\degC) in the NASA dataset,
finding that confounding concentrates at specific phases
(Fig.~\ref{fig:waveform_confounding}):
\begin{itemize}
  \item \textbf{Phase~1} ($t < 0.10$, early discharge):
        A sharp voltage drop occurs at low temperature (difference $\approx -0.12$),
        attributable to the temperature dependence of internal resistance.
  \item \textbf{Phase~2} ($0.10 \leq t \leq 0.86$, stable discharge):
        Gradual difference (difference $\approx -0.04$).
  \item \textbf{Phase~3} ($t > 0.86$, end of discharge):
        The sign reverses between low and high temperature
        (difference $= +0.21 \sim +0.36$),
        due to accelerated degradation at high temperature
        causing early terminal voltage drop.
\end{itemize}

Temperature confounding is also found to grow as cycles progress
(early: $+0.012$ $\to$ late: $-0.024$, approximately $2\times$).
Based on this finding, a timestep weight of $\times 3$ was applied
to Phase~1 and Phase~3 in the validator.

To quantitatively verify the effectiveness of DHG,
we measured the accuracy with which the original temperature condition
can be discriminated from the generated waveforms
(temperature discrimination accuracy; Table~\ref{tab:temp_discrimination}).
\begin{table*}[t]
\centering
\begin{minipage}[t]{0.47\textwidth}
\centering
\caption{Temperature discrimination accuracy
(random baseline: NASA = 0.167, MICH\_EXP = 0.333).
Values for 1-layer freeze.}
\label{tab:temp_discrimination}
\begin{tabular}{lrr}
\toprule
Dataset & Temp.\ range & Accuracy \\
\midrule
NASA      & 39\degC & 0.367 \\
MICH\_EXP & 50\degC & 0.617 \\
\bottomrule
\end{tabular}
\end{minipage}
\hfill
\begin{minipage}[t]{0.47\textwidth}
\centering
\caption{Effect of clamp upper bound on RMSE (MICH\_EXP).}
\label{tab:clamp_effect}
\begin{tabular}{lrr}
\toprule
Clamp & $\beta_{21}$ & 2-layer RMSE \\
\midrule
0.5 & 0.501 (at bound) & 0.290 \\
0.8 & 0.802 (at bound) & \textbf{0.276} \\
\bottomrule
\end{tabular}
\end{minipage}
\end{table*}
Both datasets exceed the random baseline:
NASA achieves 0.367 (2.2$\times$ random) and
MICH\_EXP achieves 0.617 (1.9$\times$ random).
This above-chance accuracy indicates that DHG is sensitive to temperature confounding structure,
consistent with its role as a \textit{diagnostic} mechanism.
The stronger detection in MICH\_EXP (50\degC\ range vs.\ 39\degC)
is qualitatively consistent with Arrhenius-type degradation kinetics:
larger temperature differences produce exponentially larger differences
in degradation rates, making confounding more detectable.
The temperature-sensitivity of the backdoor coefficients $\beta_{kj}$ and this discrimination accuracy
together constitute the primary empirical evidence that DHG responds to physical confounding structure.
Finally, relaxing the clamp upper bound of backdoor coefficients $\beta_{kj}$ from 0.5 to 0.8
improves 2-layer freeze RMSE in MICH\_EXP from 0.290 (clamp 0.5)
to 0.276 (clamp 0.8; Table~\ref{tab:clamp_effect}).
All $\beta$ values reaching the clamp upper bound
signals that larger corrections are needed,
and determining the appropriate clamp upper bound remains a future challenge.

\section{Conclusion}

We demonstrated that physics-constrained generative models can extrapolate
across a temperature range of 39\degC\ (24\degC\ $\to$ 4.0--43.0\degC)
with RMSE = 0.215---a \textbf{46\% improvement} over the unconstrained baseline
and a \textbf{2.1$\times$ improvement} over training from scratch---
by combining three principled components:
(i) target-excluded cross-domain pretraining of FNO(1),
(ii) DHG as a diagnostic and control mechanism for temperature confounding,
and (iii) Coarse-to-Fine hierarchical physical constraints.
To our knowledge, this is the first work to combine causal deconfounding
with hierarchical physics constraints in a generative flow matching framework.
Three principled findings advance the understanding of physics-guided generative learning:
(1) \textbf{Target-exclusion is a general principle, not a dataset artifact}:
pretraining on heterogeneous source conditions (NMC+LFP) without the target domain (NASA)
achieves RMSE 0.224, outperforming the NASA-included counterpart (RMSE 0.324) by 39\%.
The mechanism is explained by invariant risk minimization
(Section~\ref{sec:discussion_cross_chemistry}).
(2) \textbf{Pretraining loss is a misleading proxy for downstream performance}:
temperature diversity in pretraining data is the decisive factor,
not reconstruction accuracy---a finding that challenges common practice
in transfer learning for physical systems.
(3) \textbf{The causal structure of data governs extrapolation feasibility}:
constraints are effective for temperature extrapolation (where temperature
acts as a confounding variable with invariant structure)
but not for cycle extrapolation (where nonlinear degradation feedback dominates).
This contrast provides a causal criterion for when physics constraints help
and when they do not---a practically actionable insight for model design.
%
Table~\ref{tab:cycle_extrap} shows cycle extrapolation results (early 60\% $\to$ late 40\%).
Physical constraints improve interpolation (ID1/ID2) but W/O outperforms WITH in all OOD zones.
\begin{table}[h]
\centering
\caption{Zone-wise RMSE for cycle extrapolation (training: first 60\%; evaluation: OOD1--3).}
\label{tab:cycle_extrap}
\begin{tabular}{lrrl}
\toprule
Zone & WITH & W/O & Verdict \\
\midrule
ID1 (0--33\%)   & 0.0983 & 0.0985 & WITH better \\
ID2 (33--60\%)  & 0.1001 & 0.1020 & WITH better \\
\midrule
OOD1 (60--73\%) & 0.0173 & 0.0136 & \textbf{W/O better} \\
OOD2 (73--86\%) & 0.0179 & 0.0118 & \textbf{W/O better} \\
OOD3 (86--100\%)& 0.0200 & 0.0134 & \textbf{W/O better} \\
\bottomrule
\end{tabular}
\end{table}
We report this contrast as an actionable limitation:
future work should explicitly model degradation feedback structure
for cycle extrapolation.

Future directions include:
(1) optimization of the DHG clamp upper bound,
(2) improving validator S/N ratio via degradation\_rate input,
(3) strengthening temperature extrapolation guarantees via equivariant FNO,
(4) generalization to other battery datasets (CALCE, Oxford, Stanford), and
(5) explicit modeling of degradation feedback structure for cycle extrapolation.

\textbf{Limitation.}
Current limitations include computational cost (approximately $1.5\times$ the baseline) and restriction to the NASA dataset.

\bibliographystyle{plain}
\bibliography{hpc_fno_cfm}

\appendix
\setcounter{table}{0}
\renewcommand{\thetable}{A\arabic{table}}
\setcounter{figure}{0}
\renewcommand{\thefigure}{A\arabic{figure}}

\section{Temperature Confounding in NASA Discharge Waveforms}
\label{sec:waveform_confounding_appendix}

This appendix provides background for ML readers unfamiliar with battery electrochemistry,
explaining why temperature acts as a confounding variable in battery discharge data
and how this motivates the DHG design.

\subsection*{What is a discharge waveform?}

A lithium-ion battery produces a voltage--time curve during each discharge cycle.
As the battery releases stored energy, its terminal voltage $V(t)$ drops
from a fully charged value ($\sim$4.2\,V) to a cutoff voltage ($\sim$2.7\,V).
In this work, each waveform is resampled to 50 equally spaced timesteps
and normalized to $[0,1]$ per cycle (relative normalization),
so $t=0$ is the start of discharge and $t=1$ is the end.

\subsection*{Why does temperature affect the waveform shape?}

Temperature influences battery voltage through two distinct physical mechanisms
that operate at \textit{different timescales} and in \textit{opposite directions}:

\textbf{(1) Internal resistance effect (instantaneous, Phase~1).}
At low temperatures (e.g., 4\degC), the electrolyte viscosity increases,
slowing lithium-ion transport and raising internal resistance $R_{\text{int}}$.
By Ohm's law, the terminal voltage drops by $\Delta V = I \cdot R_{\text{int}}$
when current begins to flow.
This causes a \textit{sharp initial voltage drop} at low temperature
that disappears at high temperature---a purely temperature-driven artifact
unrelated to the battery's true degradation state.

\textbf{(2) Degradation acceleration effect (cumulative, Phase~3).}
At high temperatures (e.g., 43\degC), SEI (Solid Electrolyte Interphase) layer growth
accelerates according to the Arrhenius law: $k(T) = A\exp(-E_a / RT)$.
Over many cycles, this causes the battery to lose capacity faster at high temperature.
Near the end of discharge ($t > 0.86$), a heavily degraded battery
(one that has been cycled at high temperature) reaches its cutoff voltage earlier,
causing a \textit{premature terminal voltage drop}.
This produces the \textit{sign reversal} in Phase~3:
at high temperature, $V_{43^\circ}$ drops below $V_{4^\circ}$,
so $\Delta V = V_{4^\circ} - V_{43^\circ} > 0$.

\subsection*{Why is this a confounding problem?}

These two mechanisms create what statisticians call \textbf{confounding}:
temperature simultaneously causes changes in both the voltage waveform shape
and the apparent degradation state.
A naive model that sees waveforms at multiple temperatures cannot distinguish
``this waveform looks degraded because the battery is old''
from ``this waveform looks degraded because the temperature is high.''

Concretely, the observed correlation between internal resistance and capacity retention
in the NASA dataset is $r = -0.204$ (negative) when pooling all temperatures,
but \textit{reverses sign} to $r = +0.188$ to $+0.320$ when conditioning on temperature---
a textbook example of \textbf{Simpson's Paradox}~\cite{simpson1951interpretation}.
This means that a physical constraint trained without removing temperature confounding
would systematically penalize physically valid high-temperature waveforms,
or reward physically invalid low-temperature artifacts.

\subsection*{What Fig.~\ref{fig:waveform_confounding} shows}

Panel~(a) shows normalized discharge voltage waveforms at 4\degC\ (solid blue)
and 43\degC\ (dashed red) for NASA batteries.
The three shaded regions correspond to the three confounding phases identified in our analysis.
Panel~(b) shows the per-timestep difference $\Delta V = V_{4^\circ} - V_{43^\circ}$,
making the direction and magnitude of temperature confounding explicit at each timestep:

\begin{itemize}
  \item \textbf{Phase~1} ($t < 0.10$, early discharge, orange shading):
    $\Delta V \approx -0.12$.
    The 4\degC\ waveform starts \textit{lower} than the 43\degC\ waveform
    because high internal resistance at low temperature causes an initial voltage drop.
    This is a pure temperature artifact---the battery's true capacity is unaffected.
    DHG applies a $\times 3$ timestep weight here to increase sensitivity
    to this confounding phase.

  \item \textbf{Phase~2} ($0.10 \leq t \leq 0.86$, stable discharge, green shading):
    $\Delta V \approx -0.04$.
    The difference is small and relatively stable.
    Temperature confounding is present but modest in this region.
    The confounding magnitude approximately \textit{doubles} from early to late cycles
    (early: $+0.012$, late: $-0.024$) as cumulative degradation diverges between conditions.

  \item \textbf{Phase~3} ($t > 0.86$, end of discharge, red shading):
    $\Delta V = +0.21$ to $+0.36$.
    The sign \textit{reverses}: the 43\degC\ waveform now falls \textit{below}
    the 4\degC\ waveform because Arrhenius-accelerated degradation at high temperature
    causes early terminal voltage cutoff.
    This phase carries the strongest confounding signal,
    and DHG applies a $\times 3$ timestep weight here as well.
\end{itemize}

\subsection*{Connection to DHG}

The DHG validator $f_k$ is trained to produce high scores for physically inconsistent waveforms.
Without deconfounding, a validator exposed to mixed-temperature data would learn
to assign high inconsistency scores to \textit{any} low-temperature waveform
(because Phase~1 always has a negative dip), regardless of degradation state---
mistaking a temperature artifact for a physical violation.

By applying the do-operator $\mathrm{do}(T = T')$ over a virtual temperature set
$\mathcal{T}_{\text{cf}}$, DHG estimates what the violation score \textit{would be}
if temperature were held fixed, removing the Phase~1 and Phase~3 artifacts
from the gate signal $g_k$.
The temperature discrimination accuracy reported in Section~\ref{sec:dhg_analysis}
(NASA: 2.2$\times$ random, MICH\_EXP: 1.9$\times$ random) confirms that DHG
retains sensitivity to these confounding phases, consistent with its design
as a diagnostic mechanism.

\begin{figure}[h]
\centering
\includegraphics[width=0.82\columnwidth]{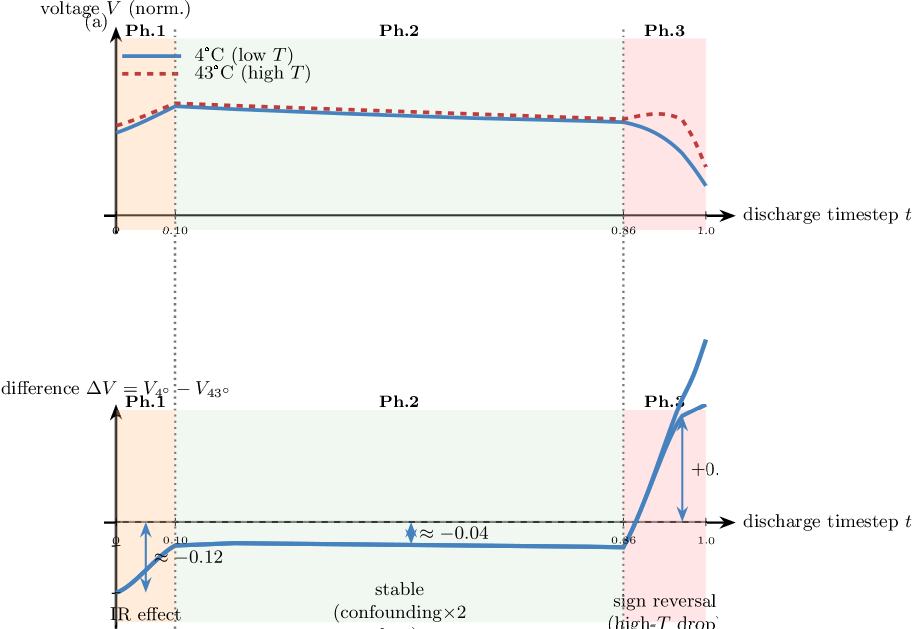}
\caption{Temperature confounding in NASA discharge waveforms (4\degC\ vs.\ 43\degC).
\textbf{(a)} Normalized voltage waveforms.
Solid blue: 4\degC\ (low temperature); dashed red: 43\degC\ (high temperature).
Three shaded regions mark Phase~1 (orange), Phase~2 (green), Phase~3 (red).
\textbf{(b)} Per-timestep difference $\Delta V = V_{4^\circ} - V_{43^\circ}$.
Phase~1 ($t<0.10$): $\Delta V \approx -0.12$ (internal resistance artifact at low $T$).
Phase~2 ($0.10\leq t \leq 0.86$): $\Delta V \approx -0.04$ (stable, modest confounding).
Phase~3 ($t>0.86$): sign reversal to $\Delta V = +0.21$--$+0.36$
(Arrhenius-accelerated degradation at high $T$ causes early terminal voltage drop).
Confounding magnitude approximately doubles from early to late cycles.}
\label{fig:waveform_confounding}
\end{figure}

\section{Supplementary Experiment: Capacity Scalar Temperature Extrapolation}

\paragraph{FNO(1) Spectral Convolution.}
The spectral convolution layer of FNO(1) is defined as:
$\mathcal{K}(u)(x) = \mathcal{F}^{-1}(R_\theta \cdot \mathcal{F}(u))(x)$,
where $\mathcal{F}$ denotes the Fourier transform and
$R_\theta \in \mathbb{C}^{d_v \times d_v \times k_{\max}}$ are learnable spectral weights~\cite{li2020fourier}.
Our implementation uses $k_{\max} = 16$ Fourier modes, width $d_v = 64$, and 3 layers.

Frozen pretrained FNO(1) achieves RMSE = 0.084,
approximately $3\times$ better than scratch (0.214)
(Table~\ref{tab:cap_extrap}).

\begin{table*}[h]
\centering
\caption{Capacity scalar temperature extrapolation results (RMSE, lower is better).
Training: 24\degC\ only (B0005--B0036, 7 batteries).}
\label{tab:cap_extrap}
\begin{tabular}{lrrrrr}
\toprule
Method & Near & Low1 & Low2 & Mean \\
\midrule
Scratch         & 0.214 & 0.246 & 0.183 & 0.214 \\
Finetune        & 0.198 & 0.231 & 0.179 & 0.203 \\
\textbf{Freeze (proposed)} & \textbf{0.082} & \textbf{0.085} & \textbf{0.086} & \textbf{0.084} \\
\bottomrule
\end{tabular}
\end{table*}

Frozen pretrained FNO(1) achieves RMSE = 0.084,
a \textbf{$\sim$3$\times$ improvement} over scratch (0.214).
The fact that Finetune (0.203) performs comparably to scratch
demonstrates that \textbf{the physical structure acquired through pretraining
must be preserved by freezing}.

\section{Theoretical Analysis}

This section presents five theorems providing the theoretical foundations of HPC-FNO-CFM.
Theorems~1 and~2 establish well-posedness of the generation process;
Theorem~\ref{thm:generalization} bounds generalization error under hierarchical physical constraints;
Theorem~\ref{thm:freeze} provides the theoretical justification for FNO freezing;
Theorem~\ref{thm:arrhenius} gives extrapolation guarantees under the temperature encoding used in this work.

\subsection{Well-Posedness of the Generation Process}

\begin{theorem}[Existence and Uniqueness of the Generation ODE]
Assume the velocity field integrating CFM with FNO guidance,
$v_\theta(x,t,c) = v_{\text{CFM}}(x,t,c) + \alpha(t)\,\widehat{\mathcal{T}}(x,c)$,
is Lipschitz continuous and bounded.
Then for any initial condition, a unique solution to the generation ODE exists
\cite{picard1890, coddington1955theory}.
\end{theorem}

\begin{proof}[Proof sketch]
When $v_{\text{CFM}}$ has Lipschitz constant $L_1$ and $\widehat{\mathcal{T}}$ has
Lipschitz constant $L_2$,
the composite velocity field has Lipschitz constant bounded by $L = L_1 + \alpha_{\max}L_2$.
By the Picard--Lindel\"{o}f theorem~\cite{coddington1955theory},
existence and uniqueness of the ODE solution are guaranteed.
\end{proof}

\begin{theorem}[Boundedness of FNO Approximation Error]
Let $\mathcal{T}$ be the true physical operator and $\widehat{\mathcal{T}}$ its FNO approximation,
with $\|\widehat{\mathcal{T}} - \mathcal{T}\|_\infty \leq \varepsilon_{\mathrm{FNO}}$.
The deviation between the FNO-guided solution $x_t$ and the ideal physical solution $y_t$ satisfies
\[
\|x_t - y_t\| \leq C(t)\,\varepsilon_{\mathrm{FNO}} + \|x_0-y_0\|e^{Lt},
\]
where $C(t)$ is a bounded function depending on time and the Lipschitz constant.
\end{theorem}

\begin{proof}[Proof sketch]
We construct the difference equation between the generated and ideal trajectories
and apply the Gronwall inequality~\cite{gronwall1919, evans2010partial}
to obtain the upper bound.
\end{proof}

\subsection{Generalization Error Reduction via Physical Constraints}

\begin{theorem}[Generalization Error Bound under Hierarchical Physical Constraints]
\label{thm:generalization}
Let $\mathcal{H}_{\mathcal{C}}$ be the function class of velocity fields satisfying
$K$-level hierarchical physical constraints
$\mathcal{C} = \{\mathcal{C}_1, \ldots, \mathcal{C}_K\}$ (with $K=3$ in this work),
and let $\mathcal{H}$ be the unconstrained function class.
For sample size $n$ and confidence $1-\delta$,
the following holds for any $h \in \mathcal{H}_{\mathcal{C}}$:
\begin{equation}
\mathbb{E}[\mathcal{L}(h)] \leq \hat{\mathcal{L}}(h)
+ \mathfrak{R}_n(\mathcal{H}_{\mathcal{C}})
+ \sqrt{\frac{\ln(1/\delta)}{2n}},
\end{equation}
where $\mathfrak{R}_n(\mathcal{H}_{\mathcal{C}})$ is the Rademacher complexity, satisfying
\begin{equation}
\mathfrak{R}_n(\mathcal{H}_{\mathcal{C}})
\leq \mathfrak{R}_n(\mathcal{H})
- \frac{1}{\sqrt{n}}\sum_{k=1}^{K} \rho_k \cdot \kappa(\mathcal{C}_k),
\end{equation}
with $\rho_k > 0$ the constraint strength at level $k$
and $\kappa(\mathcal{C}_k) \geq 0$ the complexity reduction due to constraint $k$.
\end{theorem}

\begin{proof}[Proof sketch]
Each hierarchical constraint $\mathcal{C}_k$ imposes linear inequality constraints
on the space of admissible velocity fields.
The general upper bound on Rademacher complexity for constrained function classes
follows Bartlett \& Mendelson~\cite{bartlett2002rademacher}
and Mohri et al.~\cite{mohri2018foundations}.
When constraints reduce the hypothesis space from $|\mathcal{H}|$ to
$|\mathcal{H}_{\mathcal{C}}| \leq |\mathcal{H}|$,
by Massart's lemma~\cite{massart2000}
$\mathfrak{R}_n(\mathcal{H}_{\mathcal{C}}) \leq \sqrt{2\ln|\mathcal{H}_{\mathcal{C}}|/n}$ holds,
and satisfying constraints reduces the effective dimension of the hypothesis space
by $\kappa(\mathcal{C}_k) = \ln|\mathcal{H}| - \ln|\mathcal{H}_{\mathcal{C}_k}|$.
The hierarchical design ensures that upper-level constraints (Level~1)
are most universal, so $\kappa(\mathcal{C}_1) \geq \kappa(\mathcal{C}_k),\, k>1$.
\end{proof}

\begin{corollary}
The generalization error of the proposed method (freeze + 3-level constraints)
is reduced by $O\!\left(\sum_{k=1}^{3} \rho_k \kappa(\mathcal{C}_k) / \sqrt{n}\right)$
compared to the unconstrained CFM baseline.
\end{corollary}

\subsection{Theoretical Justification for FNO Freezing}

\begin{theorem}[Information Preservation and Generalization Advantage of Frozen Representations]
\label{thm:freeze}
Let $\theta^*$ be the pretrained FNO(1) parameters,
$f_{\text{freeze}}$ the model trained with CFM only (FNO frozen),
and $f_{\text{finetune}}$ the fully fine-tuned model.
When $\theta^*$ achieves an $\varepsilon$-approximation of the
temperature-dependent physical operator $\mathcal{T}$
($\|\widehat{\mathcal{T}}_{\theta^*} - \mathcal{T}\|_\infty \leq \varepsilon$),
the following holds for the downstream task:
\begin{equation}
\mathbb{E}[\mathcal{L}(f_{\text{freeze}})]
\leq \mathbb{E}[\mathcal{L}(f_{\text{finetune}})]
+ \Delta_{\text{forget}}(\varepsilon, n_{\text{ds}}),
\end{equation}
where $\Delta_{\text{forget}} \geq 0$ is the physical knowledge loss term
due to catastrophic forgetting,
and $\Delta_{\text{forget}} > 0$ when $n_{\text{ds}} \ll n_{\text{pt}}$.
\end{theorem}

\begin{proof}[Proof sketch]
Fine-tuning minimizes $\mathcal{L}_{\text{ds}}$,
which generally conflicts with $\mathcal{L}_{\text{pt}}$
(catastrophic forgetting; \cite{mccloskey1989catastrophic,kirkpatrick2017ewc}).
When $n_{\text{ds}} \approx 500 \ll n_{\text{pt}} \approx 14{,}000$,
parameter updates cause deviation from $\theta^*$,
yielding positive probability of
$\|\widehat{\mathcal{T}}_{\theta'} - \mathcal{T}\|_\infty > \varepsilon$.
Freezing always maintains
$\|\widehat{\mathcal{T}}_{\theta^*} - \mathcal{T}\|_\infty \leq \varepsilon$,
preserving the error bound of Theorem~2.
\end{proof}

\begin{remark}
In our experiments $n_{\text{ds}} \approx 500$ vs.\
$n_{\text{pt}} \approx 14{,}000$, satisfying $n_{\text{ds}} \ll n_{\text{pt}}$.
This theoretically supports the observation that freezing
greatly outperforms fine-tuning (RMSE 0.084 vs.\ 0.203).
\end{remark}

\subsection{Temperature Extrapolation Guarantee via Temperature Encoding}

\begin{theorem}[Extrapolation Consistency of Temperature Encoding]
\label{thm:arrhenius}
Let the temperature condition be encoded as
\[
c(T) = \left[\frac{T-T_{\text{ref}}}{\sigma},\;
-\frac{E_a}{R}\left(\frac{1}{T_K}-\frac{1}{T_{\text{ref},K}}\right),\;
\exp\!\left(-\frac{E_a}{RT_K}\right)\right]^\top.
\]
For training temperature set $\mathcal{T}_{\text{train}}$ and evaluation temperature $T^*$,
even when $T^* \notin \mathcal{T}_{\text{train}}$,
if the degradation operator $\mathcal{G}: T \mapsto Q(T, \cdot)$
follows this functional form
($\mathcal{G}(T) = A \exp(-E_a/RT_K) \cdot g(\cdot)$),
then the encoding reduces the evaluation of $\mathcal{G}(T^*)$
to an interpolation problem in embedding space:
\[
c(T^*) \in \mathrm{conv}(\{c(T) : T \in \mathcal{T}_{\text{train}}\}
\cup \mathcal{B}_r(c(T^*))).
\]
\end{theorem}

\begin{proof}[Proof sketch]
The term $\exp(-E_a/RT_K)$ is monotone in $T_K$,
so embeddings of $\mathcal{T}_{\text{train}} = \{4, 10.6, 24, 38.9, 43\}$\degC\
form a monotone curve in 3-dimensional space.
Any $T^* \in [4, 43]$\degC\ lies on this curve within the convex hull
of training points (interpolation).
For $T^* \notin [4, 43]$\degC, smoothness yields a bounded extrapolation error
growing proportionally to $|T^* - T_{\text{boundary}}|$.
\end{proof}

\begin{remark}
This theorem clarifies the role of the temperature encoding:
by converting the extrapolation problem in temperature space
into an approximate interpolation problem in embedding space,
it enables extrapolation that purely data-driven models cannot achieve.
\end{remark}

\section{Proofs of Theoretical Results}

\begin{proof}[Proof of Theorem 1]
The generation ODE is
\[
\frac{dx}{dt} = v_{\text{guided}}(x,t,c) = v_\theta(x,t,c) - \mathcal{G}_{\mathrm{FNO}}(x,t).
\]
Assuming $v_\theta$ and $\mathcal{G}_{\mathrm{FNO}}$ are each Lipschitz continuous and bounded,
the composite field $v_{\text{guided}}$ is also Lipschitz continuous and bounded.
By the Picard--Lindel\"{o}f theorem, a unique solution $x(t)$ on $[0,1]$
exists for any initial condition $x(0)=x_0$.\qed
\end{proof}

\begin{proof}[Proof of Theorem 2]
Let $\Delta_t = x_t - y_t$. Then:
\[
\frac{d}{dt}\Delta_t = \big(v_\theta(x_t,t) - v_\theta(y_t,t)\big)
+ \big(\mathcal{T}(y_t) - \widehat{\mathcal{T}}(x_t)\big).
\]
By Lipschitz continuity of $v_\theta$:
$\|v_\theta(x_t,t) - v_\theta(y_t,t)\| \leq L\|\Delta_t\|$.
By FNO approximation error:
$\|\mathcal{T}(y_t) - \widehat{\mathcal{T}}(x_t)\| \leq \varepsilon_{\mathrm{FNO}} + L_f\|\Delta_t\|$.
Applying the Gronwall inequality:
\[
\|\Delta_t\| \leq \|x_0-y_0\| e^{(L+L_f)t}
+ \frac{\varepsilon_{\mathrm{FNO}}}{L+L_f}\big(e^{(L+L_f)t}-1\big).\qed
\]
\end{proof}

\section{Physical Feature Distance (PFD): Definition}
\label{sec:pfd_definition}

This section provides the formal definition of the Physical Feature Distance (PFD),
an evaluation metric proposed in this work.
PFD applies the Fr\'{e}chet Inception Distance (FID)~\cite{heusel2017gans}
to physics-constrained generative models for time series.

\paragraph{Background: FID}

FID is widely used as an evaluation metric for image generative models,
defining the distance between the real distribution $p_{\text{real}}$
and generated distribution $p_{\text{gen}}$ as the Fr\'{e}chet distance in feature space:
\begin{equation}
\text{FID} = \|\mu_r - \mu_g\|^2 +
\mathrm{Tr}\!\left(\Sigma_r + \Sigma_g
- 2(\Sigma_r\Sigma_g)^{1/2}\right),
\label{eq:fid}
\end{equation}
where $(\mu_r, \Sigma_r)$ and $(\mu_g, \Sigma_g)$ are the mean and covariance
of feature vectors of real and generated data, respectively.

\paragraph{Definition of PFD}

For time series $x \in \mathbb{R}^{L \times d}$,
we use the physically interpretable feature extractor
$\phi: \mathbb{R}^{L \times d} \to \mathbb{R}^{5}$:
\begin{equation}
\phi(x) = \left[
  \bar{x},\;
  \sigma(x),\;
  \Delta\bar{x},\;
  \max(x) - \min(x),\;
  \mathrm{slope}(x)
\right]^\top,
\label{eq:pfd_feature}
\end{equation}
where $\bar{x}$ is the time-series mean, $\sigma(x)$ the standard deviation,
$\Delta\bar{x}$ the difference between second-half and first-half means
($\bar{x}_{[L/2:]} - \bar{x}_{[:L/2]}$),
and $\mathrm{slope}(x)$ the slope of a linear regression over the series.
These features represent the global energy level, variability,
and discharge trend of the voltage waveform.

For each temperature condition $c$,
mean and covariance are estimated from
$\{\phi(x^{(i)}_{\text{real}})\}_{i=1}^{N}$ and
$\{\phi(\hat{x}^{(i)})\}_{i=1}^{N}$,
then $\mathrm{PFD}(c)$ is computed via Eq.~\eqref{eq:fid}.
The final PFD is averaged over all conditions:
\begin{equation}
\mathrm{PFD} = \frac{1}{|C|}\sum_{c \in C} \mathrm{PFD}(c).
\label{eq:pfd_final}
\end{equation}

\paragraph{Properties of PFD}

PFD provides an evaluation axis independent of RMSE.
While RMSE measures point-wise estimation error per sample,
PFD measures the statistical agreement of physical features
across the entire generated distribution.
Experiments confirm high correlation between the two (Table~\ref{tab:pfd}).
Lower PFD indicates that the statistical properties of generated voltage waveforms
are closer to real measurements.

\end{document}